\documentclass{article}
\usepackage{ijcai18}

\usepackage{times}
\usepackage{xcolor}
\usepackage{soul}
\usepackage[utf8]{inputenc}
\usepackage[small]{caption}

\usepackage{epsfig}
\usepackage{lipsum,graphicx,subcaption}
\usepackage{multicol}
\usepackage{amsmath,amsthm,amssymb,amsfonts}
\usepackage{amssymb}
\usepackage{bm}
\usepackage[noend]{algpseudocode}
\usepackage{algorithmicx,algorithm}
\newtheorem{definition}{Definition}
\newtheorem{assumption}{Assumption}

\newtheorem{theorem}{Theorem}  
\begin{document}\sloppy
\def\x{{\mathbf x}}
\def\L{{\cal L}}

\pagestyle{empty}

\title{Information Scaling Law of Deep Neural Networks}

\author{Xiao-Yang Liu}
%\author{Xiao-Yang Liu$^{2}$, Qingyun Han$^{1}$, Feng Qian$^{1}$$\\ 
%$^1$Dept. of Electrical Engineering, Columbia University\\
%$^2$The Center for Information %Geoscience, University of Electronic Science and Technology of China\\
%}
\maketitle

%%%%%%%%%%%%%%%%%%%%%%%%%%%%%%%%%%%%%
\begin{abstract}
With the rapid development of Deep Neural Networks (DNNs), various network models that show strong computing power and impressive expressive power are proposed. However, there is no comprehensive informational interpretation of DNNs from the perspective of information theory. Due to the nonlinear function and the uncertain number of layers and neural units used in the DNNs, the network structure shows nonlinearity and complexity. With the typical DNNs named Convolutional Arithmetic Circuits (ConvACs), the complex DNNs can be converted into mathematical formula. Thus, we can use rigorous mathematical theory especially the information theory to analyse the complicated DNNs. In this paper, we propose a novel information scaling law scheme that can interpret the network's inner organization by information theory. First, we show the informational interpretation of the activation function. Secondly, we prove that the information entropy increases when the information is transmitted through the ConvACs. Finally, we propose the information scaling law of ConvACs through making a reasonable assumption.
 
\end{abstract}

%\begin{keywords}
%xxxx
%\end{keywords}

%%%%%%%%%%%%%%%%%%%%%%%%%%%%%%%%%%%%%%

\section{Introduction}
\label{sec:intro}

Recently, DNNs have achieved great success \cite{Lecun2015Deep,He2016Deep,ijcai2017-497}, and there are many related works trying to explain DNNs \cite{Setiono1995Understanding,DBLP:conf/ijcai/MoskvinaL16,NIPS2017_7201}.  However, there is no comprehensive informational interpretation of DNNs, which has promising applications in the optimization process and the internal organization of DNNs called \textquotedblleft black boxes\textquotedblright \  \cite{Alain2016Understanding}. Due to the nonlinear function and the uncertain number of layers and neural units used in the DNNs, the network structure shows nonlinearity and complexity. So, it is difficult to propose the interpretation of DNNs from the perspective of information theory \cite{Shwartz2017Opening}.

There are three approaches to solve the above problems. The first approach is the Vapnik-Chervonenkis (VC) dimension \cite{Vapnik2015On}, which represents the largest amount of samples in a dataset divided by a hypothesis space. The VC dimension shows a hypothesis space's powerfulness \cite{Vapnik2000The}. In \cite{Vapnik1994Measuring}, the authors suggest that we should determine the VC dimension by experience, but they also state that the approach described cannot be used in neural networks because they are \textquotedblleft beyond theory\textquotedblright. So far, the VC dimension can only be used to approximate neural networks. The second approach called capacity scaling law \cite{Friedland2017A} predicts the behavior of perceptron networks by calculating the lossless memory (LM) dimension and Mackey (MK) dimension \cite{MacKay2003Information}. However, this approach is based on an ideal neural network and only applies to perceptron networks. The third approach named  Information Bottleneck (IB) theoretical bound is a technique in information theory introduced by \cite{Naftali2000The}. In \cite{Shwartz2017Opening}, the authors demonstrate the effectiveness of the Information Plane visualization of DNNs by the IB theoretical bound. However, the IB bound doesn't have good performance on more complex and deeper networks. So, it is important to propose an interpretation theory that can apply to all DNNs and has more precise rules of measurement. 

To address the above difficulties, we introduce the typical network models of the DNNs named ConvACs \cite{Cohen2016On}. The ConvACs can be viewed as convolutional networks that calculate high-dimensional functions by tensor decompositions \cite{Cohen2017Analysis}. Thus, we can map the ConvACs to concrete mathematical formulas through the decomposited high-dimensional functions. Obviously, it is convenient for us to analyze the mathematical formulas by mathematical method. That is, we can analyze the complex DNNs by mathematical method, especially the information theory, which allows us to understand the complicated DNNs more concretely and provides us designing guidelines of the DNNs.

In this paper, we propose a novel information scaling law that can interpret the network from the perspective of information theory, which provides a better understanding of the expressive efficiency of the DNNs. Firstly, we show the informational interpretation of activation function. Secondly, we prove that the Information entropy increases when the information is transmitted through the ConvACs. Finally, we propose the information scaling law of ConvACs through making a reasonable assumption.

The remainder of this paper is organized as follows. Section 2 provides an introduction of tensor theory and ConvACs. Section 3 describes the informational interpretation of DNNs and demonstrates that the information entropy increases through ConvACs. In Section 4, we propose the information scaling law of ConvACs. We conclude in Section 5.

% %%%%%%%%%%%%%%%%%%%%%%%%%%%%%%%%%%%%%%
\section{Preliminaries}

\subsection{Tensor Theory}
 A tensor can be regarded as a multi-dimensional array $\mathcal{A}_{d_1,...,d_N} \in \mathbb{R}$, where $d_i\in[M_i],~i\in[N]$ and $[N]$ denotes the set $\{1,2,...,N\}$. We call the number of indexing entries as modes, which is known as the order of the tensor. The number of values that each mode can take is known as the dimension of the mode. So, the tensor $\mathcal{A} \in \mathbb{R}^{M_1\otimes...\otimes M_N}$ mentioned before is of  order $N$ with dimension $M_i$ in $i$-th mode. In this paper, we consider that $M_1 = ...=M_N=M$, i.e. $\mathcal{A}\in(\mathbb{R}^M)^{\otimes N}$. The tensor product denoted by $\otimes $ is the basic operator in tensor analysis, which is the generalization of outer product of the vectors to tensors. 

The main methods which we use from tensor theory in our paper is tensor decompositions .The rank-1 decomposition is the most common tensor decomposition, which is viewed as a CANDECOMP/PARAFAC (CP) decomposition. Similar to the matrix decoposition, the rank-$Z$ CP decomposition of a tensor ${\mathcal{A}}$ can be represented as:
\begin{small}
\begin{equation}
\begin{split}
{\mathcal{A}} &= \sum_{z=1}^{Z}a_z\bm{\mathrm{a}}^{z,1}\otimes\cdots\otimes\bm{\mathrm{a}}^{z,N,} \\
\Rightarrow\mathcal{A}_{d_1,...,d_N}&=\sum_{z=1}^{Z}a_z\prod_{i=1}^{N}a_{d_i}^{z,i},
\end{split}
\end{equation}
\end{small}
where $\{\bm{\mathrm{a}}^{z,i} \in \mathbb{R}^{M_i}\}_{i=1,z=1}^{N,Z}$ and $\bm{\mathrm{a}} \in \mathbb{R}^Z$ are the parameters. 

Another decomposition is hierarchical and known as the Hierarchical Tucker (HT) decomposition \cite{Hackbusch2014Tensor}. Unlike the CP decomposition combining vectors into higher order tensors with only one step, the HT decomposition follows a tree structure. It combines vectors into matrices, and combines these matrices into 4th ordered tensor and so on recursively in a hierarchically way. Specifically, the recursive formula of the HT decomposition for a tensor $\mathcal{A} \in (\mathbb{R}^M)^{\otimes N}$ is described as follows, where $N = 2^L$,
\begin{small}
 \begin{equation}
 \begin{split}
\phi^{1,j,\gamma} &= \sum_{\alpha = 1}^{r_0}a_{\alpha}^{1,j,\gamma}\bm{\mathrm{a}}^{0,2j-1,\alpha}\otimes\bm{\mathrm{a}}^{0,2j,\alpha}, \\
&\cdots\\
\phi^{l,j,\gamma} &= \sum_{\alpha = 1}^{r_{l-1}}a_{\alpha}^{l,j,\gamma}\underbrace{\phi^{l-1,2j-1,\alpha}}_{\mathrm{order}\, 2^{l-1}}\otimes\underbrace{\phi^{l-1,2j,\alpha}}_{\mathrm{order}\,2^{l-1}}, \\
&\cdots\\
\phi^{L-1,j,\gamma} &= \sum_{\alpha = 1}^{r_{L-2}}a_{\alpha}^{L-1,j,\gamma}\underbrace{\phi^{L-2,2j-1,\alpha}}_{\mathrm{order}\, \frac{N}{4}}\otimes\underbrace{\phi^{L-2,2j,\alpha}}_{\mathrm{order}\,\frac{N}{4}}, \\
\mathcal{A} &= \sum_{\alpha=1}^{r_{L-1}}a_{\alpha}^L\underbrace{\phi^{L-1,1,\alpha}}_{\mathrm{order}\,\frac{N}{2}}\otimes\underbrace{\phi^{L-1,2,\alpha}}_{\mathrm{order}\,\frac{N}{2}},
 \end{split}
 \end{equation}
 \end{small}
where the parameters of the decomposition are the vectors $\{\mathrm{a}^{l,j,\gamma} \in \mathbb{R}^{r_{l-1}}\}_{l\in \{0,...,L-1\},j \in [{N}/{2^l}], \gamma \in [r_l]}$ and the top level vector $\mathrm{a}^L \in \mathbb{R}^{r_{L-1}}$, and the scalars $r_0,...,r_{L-1} \in \mathbb{N}$ are referred to as the ranks of the decomposition. Similar to the CP decomposition, any tensor can be converted to an HT decomposition by only a polynomial increase in the number of parameters.
\subsection{Convolutional Arithmetic Circuits}

In 2016, Cohen and Sharir raised a family of models called Convolutional Arithmetic Circuits (ConvACs) \cite{Cohen2016On}. The ConvACs are convolutional networks which have a particular choice of non-linearities. They set point-wise activations to be linear (as opposed to sigmoid), and the pooling operators are based on product (as opposed to max or average). In fact, The ConvACs is related to many mathematical fields (tensor analysis, measure theory, functional analysis, theoretical physics, graph theory and more), which makes them especially amendable to theoretical analysis. 
\begin{figure}[t]
	\centering
	\includegraphics[totalheight=4.cm]{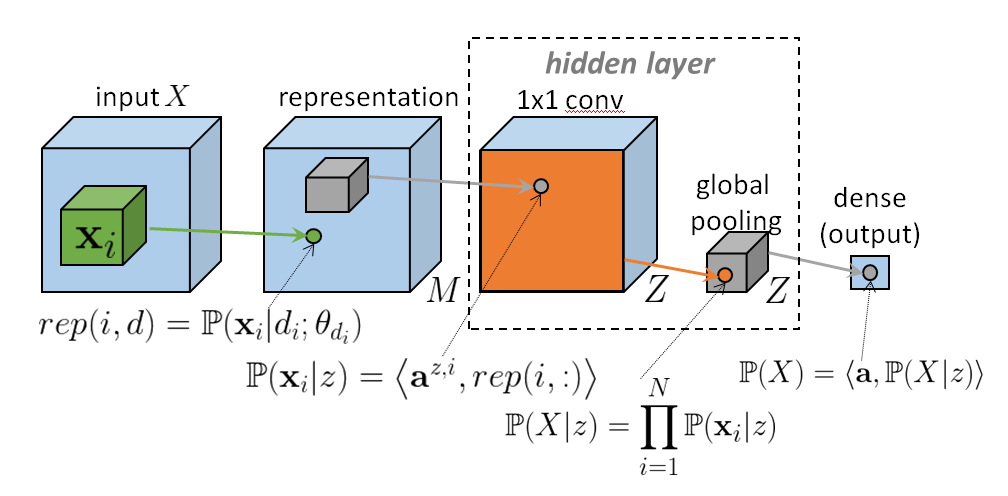}
	\caption{The CP model carried out by a shallow Convolutional Arithmetic Circuits.}
	\label{fig1}
	\vspace{-0.17in}
\end{figure}
In \cite{Cohen2016Deep}, Cohen represents the input signal $X$ by a sequence of low-dimensional local structures $X = (\bm{\mathrm{x}}_1,...,\bm{\mathrm{x}}_N) \in (\mathbb{R}^s)^N$. $X$ is typically considered as an image, where each local structure $\bm{\mathrm{x}}_i$ corresponds to a local patch from that image. Mixture models is one of the simplest forms of ConvACs. The probability distribution of mixture models is defined by the convex combination of $M$ mixing components $\{\mathbb{P}(\bm{\mathrm{x}}|d; \theta_d)\}_{d=1}^{M}$ (e.g. Normal distributions): $\mathbb{P}(\bm{\mathrm{x}}) = \sum_{d=1}^M\mathbb{P}(d)\mathbb{P}(\bm{\mathrm{x}}|d;\theta_d)$. It is easy to learn mixture models, and many of them can be used to approximate any probability distribution if given enough numbers of components, which makes them suitable for various tasks.

Formally, for all $i\in[N]$ there exists $d_i\in[M]$ such that $\bm{\mathrm{x}}_i \sim \mathbb{P}(\bm{\mathrm{x}}|d_i;\theta_{d_i})$, where $d_i$ is a hidden variable presenting the matching component of the $i$-th local structure. So, the probability density of sampling $X$ is described by: 
\begin{equation}
\mathbb{P}(X) = \sum_{d_1,...,d_N = 1}^{M} \mathbb{P}(d_1,...,d_N)\prod_{i = 1}^{N}\mathbb{P}(\bm{\mathrm{x}}_i|d_i;\theta_{d_i}), \label{xxx}
\end{equation} where $\mathbb{P}(d_1,...,d_N)$ represents the prior probability of assigning components $d_1,...,d_N$ to their rspective local structures $\bm{\mathrm{x}}_1,...,\bm{\mathrm{x}}_N$. As with typical mixture models, any probability density function $\mathbb{P}(X)$ could be approximated arbitrarily well by (\ref{xxx}) as $M \to \infty$. The prior probability $ \mathbb{P}(d_1,...,d_N)$ can be represented by a tensor $\mathcal{A}_{d_1,...,d_N}$, which is of order N with dimension M in each mode.

\section{Information Analysis of ConvACs}
We analyze the ConvACs by the information theory. The analysis include the activation function and network's structure. With the tensor algebra, the ConvACs can be used to represent both shallow and deep networks. 
\subsection{Information Theory}

 We will need the following three definitions to calculate the entropy \cite{Jacobs2007Elements}.

\begin{definition}
Conditioning does not increase entropy,
\begin{equation}
H(X|Y) \leq H(X),
\end{equation}
with equality if and only if X and Y are independent of each other. Where $H(X)$ and $H(X|Y)$ are the entropy and conditonal entropy of $X$ and $Y$, respectively.
\end{definition}

\begin{definition}
Independence bound on information entropy. Given $n$ random variables $X_1,X_2,...,X_n$, with a joint distribution $p(x_1,x_2,...,x_n)$, then
\begin{equation}
H(X_1,X_2,...X_n) \leq \sum_{i=1}^n{H(X_i)},
\end{equation}
with equality if and only if $X_i$ are independent of each other. Where  $H(X_i)$ and $H(X_1,X_2,...X_n)$ are the entropy and joint entropy of $X_i$ and $(X_1,X_2,...X_n)$, respectively.
\end{definition}

\begin{definition}
Mutual Information.
Given any two random variables, $X$ and $Y$, with a joint distribution $p(x,y)$, their Mutual Information is defined as:
\begin{equation}
\begin{split}
I(X;Y) 
%&= D_{KL}[p(x,y)||p(x)p(y)] \\
&= \sum_{x \in X,y \in Y}p(x,y)\log(\frac{p(x,y)}{p(x)p(y)}) \\
%&= 
%\sum_{x \in X,y \in Y}p(x,y)log(\frac{p(x|y)}{p(x)}) \\
&= H(X) - H(X|Y),
\end{split}
\end{equation}
where 
%$D_{KL}[p||q]$ is the Kullback-Liebler divergence of the distributions $p$ and $q$, %and 
$H(X)$ and $H(X|Y)$ are the entropy and conditonal entropy of $X$ and $Y$, respectively.
\end{definition}

\subsection{Activation Function Analysis}

First, we analyze how activation function affects information entropy. We come up with two theorems by analyzing the sigmoid function and the rectified linear unit (ReLU) function \cite{He2015Delving}.
\begin{theorem}
Sigmoid function reduces the information entropy.
Given the probability density function $f_x(x)$ of $X$, the sigmoid function $y = \frac{1}{1+e^{-x}}$, we can get:
\begin{equation}
H(Y) < H(X),
\end{equation}
where $H(X)$, $H(Y)$ are the information entropy of $X$ and $Y$, respectively.
\end{theorem}
\begin{proof}
By the \emph{probability Theory}, we can get the probability density function $f_y(y)$ of $Y$
\begin{equation}
f_y(y) = f_x(h(y))|h'(y)|,
\end{equation}
where $h(y) = \ln(\frac{y}{1-y})$ is the inverse function. Then, the entorpy of $Y$:
%\begin{small}
\begin{equation}
\begin{split}
H(Y)\! &=\! - \int_yf_x(y)\log(f_y(y))dy \\
\!&=\! -\int_y|h'(y)|f_x(h(y))\log(|h'(y)|f_x(h(y))dy \\
\!&=\! -\int_xf_x(x)\log f_x(x)d_x\! -\!\int_xf_x(x)\log\frac{(1\!+\!e^{-x})^2}{e^{-x}}dx \\
\!&=\! H(X) + C \\
\!&<\! H(X), \nonumber
\end{split}
\end{equation}
%\end{small}
where $C\! =\! -\int_xf_x(x)\log\frac{(1+e^{-x})^2}{e^{-x}}dx < 0$, and $H(X)$, $H(Y)$ are the information entropy of $X$ and $Y$, respectively.
\end{proof}

\begin{theorem}
ReLU function doesn't change the information entropy when used in probability models. Given the probability density function $f_X(x)$ of $X$, the ReLU function $y = max\{0,x\}$, we can get:
\begin{equation}
H(Y) = H(X),
\end{equation}
where $H(X)$, $H(Y)$ are the information entropy of $X$ and $Y$, respectively.
\end{theorem}
\begin{proof}
 Our models are probability models, the $X$ must be nonnegative number. So the ReLU function could be turn into $y\! =\! x$, which is linear and doesn't change the information entropy. 
 \end{proof}
 \begin{figure*}[t]
 	\centering
 	\includegraphics[totalheight= 6cm]{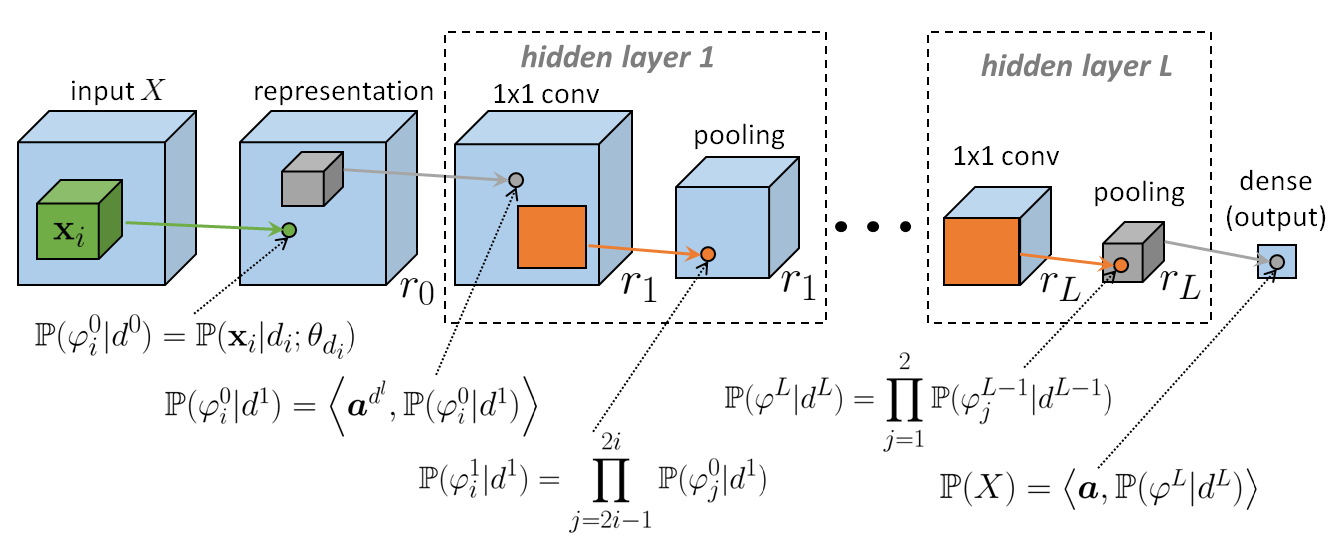}
 	\caption{The HT-model carried out by ConvACs.}
 	\label{fig5}
 	\vspace{-0.17in}
 \end{figure*}

\subsection{CP-Model Analysis}

Secondly, by representing the priors tensor according to the CP decomposition, we can get the following equation called CP-model \cite{Sharir2017Tractable}:
\begin{equation}
\begin{split}
\mathbb{P}(X) &= \sum_{d_1,...,d_N}^M(\sum_{z=1}^Z{a_z}\prod_{i=1}^N{a_{d_i}^{z,i}})\prod_{i=1}^{N}\mathbb{P}(\bm{\mathrm{x}}_i|d_i;\theta_{d_i}) \\ 
&= \sum_{z=1}^{Z}{a_z}\prod_{i=1}^N\sum_{d_i=1}^M{a_{d_i}^{z,i}\mathbb{P}(\bm{\mathrm{x}}_i|d_i;\theta_{d_i})}.     \label{4}
\end{split}
\end{equation}
As shown in Fig. \ref{fig1}, the CP-model is a shallow network. The first layer is a representation layer followed by a $1\! \times\! 1 $ convolutional layer, and then a global pooling layer, finally is the output layer. 
For CP-model, (\ref{4}) can be decomposed into three equations: 
\begin{equation}
\mathbb{P}(X)  =  \sum_{z=1}^Z{a_z}\mathbb{P}(X|z), \label{1}
\end{equation}
\begin{equation}
\mathbb{P}(X|z)  =  \prod_{i=1}^N\mathbb{P}(\bm{\mathrm{x}}_i|z), \label{2}
\end{equation}
\begin{equation}
\mathbb{P}(\bm{\mathrm{x}}_i|z)  =  \sum_{d_i=1}^M{a_{d_i}^{z,i}\mathbb{P}(\bm{\mathrm{x}}_i|d_i;\theta_{d_i})}.\label{3}
\end{equation}
That is to say, the CP-model is composed of two mixture models (\ref{1})(\ref{3}) and a continued product (\ref{2}). 
At first, we introduce a definition about the continued product.
\begin{definition}
The continued product doesn't change information entropy. i.e. for following equation:
\begin{equation}
\mathbb{P}(X) = \prod_{i=1}^{N}\mathbb{P}(\bm{\mathrm{x}}_i),
\end{equation}
the relationship of information entropy is:
\begin{equation}
H(X) = \sum_{i=1}^{N}H(\bm{\mathrm{x}}_i),
\end{equation}
where $X\! =\! (\bm{\mathrm{x}}_1,...,\bm{\mathrm{x}}_N)$, $H(X)$, $H(\bm{\mathrm{x}}_i)$ are the information entropy of $X$ and $\bm{\mathrm{x}}_i$, respectively.
\end{definition}
According to Definition 1, Definition 4 and Theorem 3 (see Section 4), (\ref{1})(\ref{2})(\ref{3}) can be transformed into the following presentations of information entropy:
 \begin{equation}
 %\begin{split}
 H(X) \geq H(X|z),
 \end{equation}
 \begin{equation}
 H(X|z) = \sum_{i=1}^N{H(\bm{\mathrm{x}}_i|z)},
 \end{equation}
 \begin{equation}
 H(\bm{\mathrm{x}}_i|z) \geq H(\bm{\mathrm{x}}_i|d).
 %\end{split}
 \end{equation}
 By combining above three formulas, we can get the total impact of CP-model on information entropy:
 \begin{equation}
 H(X) \geq H(X|z) = \sum_{i=1}^N{H(\bm{\mathrm{x}}_i|z)} \geq \sum_{i=1}^N{H(\bm{\mathrm{x}}_i|d)}. \label{49}
 \end{equation} 

\subsection{HT-Model Analysis}
Finally, we analyze the CP-model, which is a shallow network. By the HT decomposition (3), we can get the deep network called HT-model shown in Fig. \ref{fig5} \cite{Sharir2017Tractable}. And the HT-model can be presented by following equation:
 \begin{equation}
 \begin{split}
 \mathbb{P}(\varphi_i^0|d^0)  &=  \mathbb{P}(\bm{\mathrm{x}}_i|d_i;\theta_{d_i}), \\
  \mathbb{P}(\varphi_i^1|d^1)  &=  \prod_{j=2i-1}^{2i}\sum_{\alpha=1}^{r_0}a_{\alpha}^{d^{1}}\mathbb{P}(\varphi_{j}^{0}|d^{0}), \\
 &\cdots \\
  \mathbb{P}(\varphi_i^l|d^l)  &=  \prod_{j=2i-1}^{2i}\sum_{\alpha=1}^{r_{l-1}}a_{\alpha}^{d^{1}}\mathbb{P}(\varphi_{j}^{l-1}|d^{l-1}), \\
   &\cdots\\
  \mathbb{P}(\varphi^L|d^L)  &=  \prod_{j=1}^{2}\sum_{\alpha=1}^{r_{L-1}}a_{\alpha}^{d^{L}}\mathbb{P}(\varphi_{j}^{L-1}|d^{L-1}), \\  
  \mathbb{P}(X)  &= \sum_{\alpha=1}^{r_{L}}a_{\alpha}\mathbb{P}(\varphi^{L}|d^{L}), 
 \end{split}
 \end{equation}
where $d^l \in [r_{l}]$,$r_l \in \mathbb{N}, l\in\{0,...,L\}, L=\log_2N$, $\mathbb{P}(\varphi_i^l|d^l) $ presents the $i$-th component in the $l$-th layer.
 
Each layer in the HT-model can be seen as a CP-model ($N=2$). We can get following equations by the conclusion in last subsection:
\begin{equation}
\begin{split}
H(\varphi_i^0|d^0) &= H(\bm{\mathrm{x}}_i|d_i),\\
H(\varphi_i^{1}|d^1) &\geq H(\varphi_{2i-1}^{0}|d^0) + H(\varphi_{2i}^{0}|d^0), \\
&\cdots\\
H(\varphi_i^{l}|d^l) &\geq H(\varphi_{2i-1}^{l-1}|d^0) + H(\varphi_{2i}^{l-1}|d^0),\\
&\cdots\\
H(\varphi^{L}|d^L) &\geq H(\varphi_{1}^{L-1}|d^{L-1}) + H(\varphi_{2}^{L-1}|d^{L-1}),\\
H(X) &\geq H(\varphi^{L}|d^L),
\end{split}
\end{equation}
where $H(\varphi_i^{l}|d^l)$ represents the entropy of the each channel ($d^l$) of each layer ($l$), ${l\in \{0,...,L\},i \in [{N}/{2^l}], d^l \in [r_l]}$. We naturally derive the following equation:
\begin{equation}
H(X) \geq \sum_{i=1}^{N}H(\bm{\mathrm{x}}_i|d). \label{50}
\end{equation}
Combining (\ref{49}) and (\ref{50}), we prove that the information entropy increases when information is transmitted through the ConvACs. In next section, we will focus on the concrete information scaling law.

% %%%%%%%%%%%%%%%%%%%%%%%%%%%%%%%%%%%%%%%
\section{Information Scaling Law of ConvACs}

Now, we prove that the loss of information exists in the ConvACs. In this section, we will propose the information scaling law of the ConvACs. 

For the graphical description of HT-model shown in Fig. 3, the information loss could be analysized as followling. We call the operation (the dotted box shown in Fig. \ref{fig2}) as a fusion. A fusion can be presented as follows.
\begin{equation}
\mathbb{P}(\varphi_i^l|d^l) = \mathbb{P}(\varphi_{2i-1}^{l-1}|d^{l}) \cdot \mathbb{P}(\varphi_{2i}^{l-1}|d^{l}), \label{9}
\end{equation}
\begin{equation}
\mathbb{P}(\varphi_{2i-1}^{l-1}|d^{l}) = \sum_{\alpha=1}^{r_{l-1}}a_{\alpha}^{d^{l}}\mathbb{P}(\varphi_{2i-1}^{l-1}|d^{l-1}), \label{10}
\end{equation}
\begin{equation}
\mathbb{P}(\varphi_{2i}^{l-1}|d^{l}) = \sum_{\alpha=1}^{r_{l-1}}a_{\alpha}^{d^{l}}\mathbb{P}(\varphi_{2i}^{l-1}|d^{l-1}),  \label{8}
\end{equation}
where $d^l \in [r_{l}]$,$r_l \in \mathbb{N}$, $\mathbb{P}(\varphi_i^l|d^l)$ presents the $i$-th component in the $l$-th layer.
 \begin{figure}[t]
 	\centering
 	\includegraphics[totalheight=5cm]{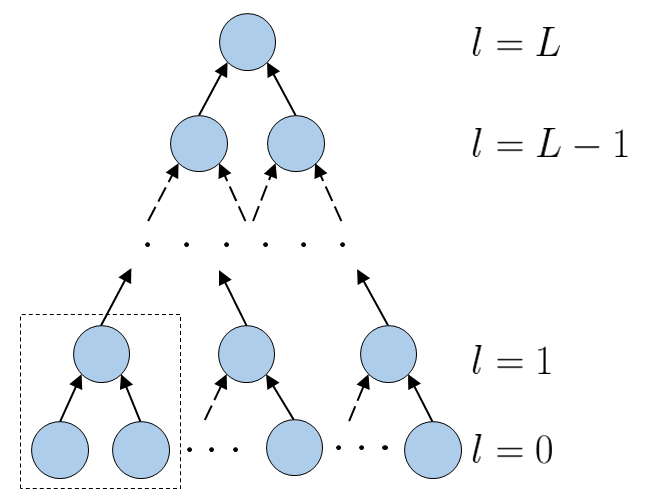}
 	\caption{Graphical description of HT-model. The $i$-th node in $l$-th layer presents $\mathbb{P}(\varphi_i^l|d^l)$.}
 	\label{fig2}
 	\vspace{-0.17in}
 \end{figure}

 \begin{theorem}
 The operation we call a mapping of probability space is shown in Fig. \ref{fig4}, the relationship of information entropy between $\varphi_{2i-1}^{l-1}|d^{l} $ and $\varphi_{2i-1}^{l-1}|d^{l-1} $ can be described by following fomular:\label{100}
 \begin{equation}
H(\varphi_{2i-1}^{l-1}|d^{l}) \geq  H(\varphi_{2i-1}^{l-1}|d^{l-1}). \label{14}
 \end{equation}
 \end{theorem}
 \begin{proof}
  As shown in Fig. \ref{fig4}, we have further refined the dotted box shown in Fig. \ref{fig3}, i.e. the (\ref{10}). The operation can be wrote as the following formula:
 \begin{equation}
 \bm{\mathrm{u}} = \bm{\mathrm{A}} \cdot \bm{\mathrm{v}}, \label{11}
 \end{equation}
 where $$\bm{\mathrm{u}}\! =\! 
 \begin{pmatrix}
 \mathbb{P}(\varphi_{2i-1}^{l-1}|{d^{l} = 1})\\
 \mathbb{P}(\varphi_{2i-1}^{l-1}|{d^{l} = 2})\\
 ...\\
 \mathbb{P}(\varphi_{2i-1}^{l-1}|{d^{l} = r_l)}\\
 \end{pmatrix},
 \bm{\mathrm{v}}\! =\!
 \begin{pmatrix}
  \mathbb{P}(\varphi_{2i-1}^{l-1}|{d^{l-1} = 1})\\
  \mathbb{P}(\varphi_{2i-1}^{l-1}|{d^{l-1} = 2})\\
  ...\\
  \mathbb{P}(\varphi_{2i-1}^{l-1}|{d^{l-1} = r_{l-1})}\\
  \end{pmatrix}$$

  \noindent$\bm{\mathrm{A}} = [\bm{a}^{d^l=1} \ \bm{a}^{d^l=2} \ ... \  \bm{a}^{d^l=r_l}]^T \in \mathbb{R}^{r_l \times r_{l-1}}$ and $\bm{a}^{d^l} \in \mathbb{R}^{r_{l-1}}$. 
 From (\ref{11}), the following fomular is obvious:
 \begin{equation}
 \begin{split}
 \mathbb{P}(\varphi_{2i-1}^{l-1}|d^{l}\!=\!j) &= \bm{a}^{d^l=j} \cdot \bm{\mathrm{v}}\\
 &= \sum_{\alpha=1}^{r_{l-1}}a_{\alpha}^{d^{l}=j}\mathbb{P}(\varphi_{2i-1}^{l-1}|d^{l-1}\!=\!\alpha). \label{12}
 \end{split}
 \end{equation}
 By the Definition 1, from (\ref{12}) we can get 
  \begin{equation}
 H(\varphi_{2i-1}^{l-1}|d^{l}\!=\!j) \geq H(\varphi_{2i-1}^{l-1}|d^{l-1}),
 \end{equation}
 and by the definition of information entropy:
 \begin{equation}
 \begin{split}
 H(\varphi_{2i-1}^{l-1}|d^{l}) &=\sum_{j=1}^{r_l}\mathbb{P}(d^{l}\!=\!j)H(\varphi_{2i-1}^{l-1}|d^{l}\!=\!j)\\
 &\geq \sum_{j=1}^{r_l}\mathbb{P}(d^{l}\!=\!j)H(\varphi_{2i-1}^{l-1}|d^{l-1}) \\
 &= H(\varphi_{2i-1}^{l-1}|d^{l-1}).
 \end{split}
 \end{equation}
 \end{proof}
  \begin{figure}[t]
  	\centering
  	\includegraphics[totalheight=5cm]{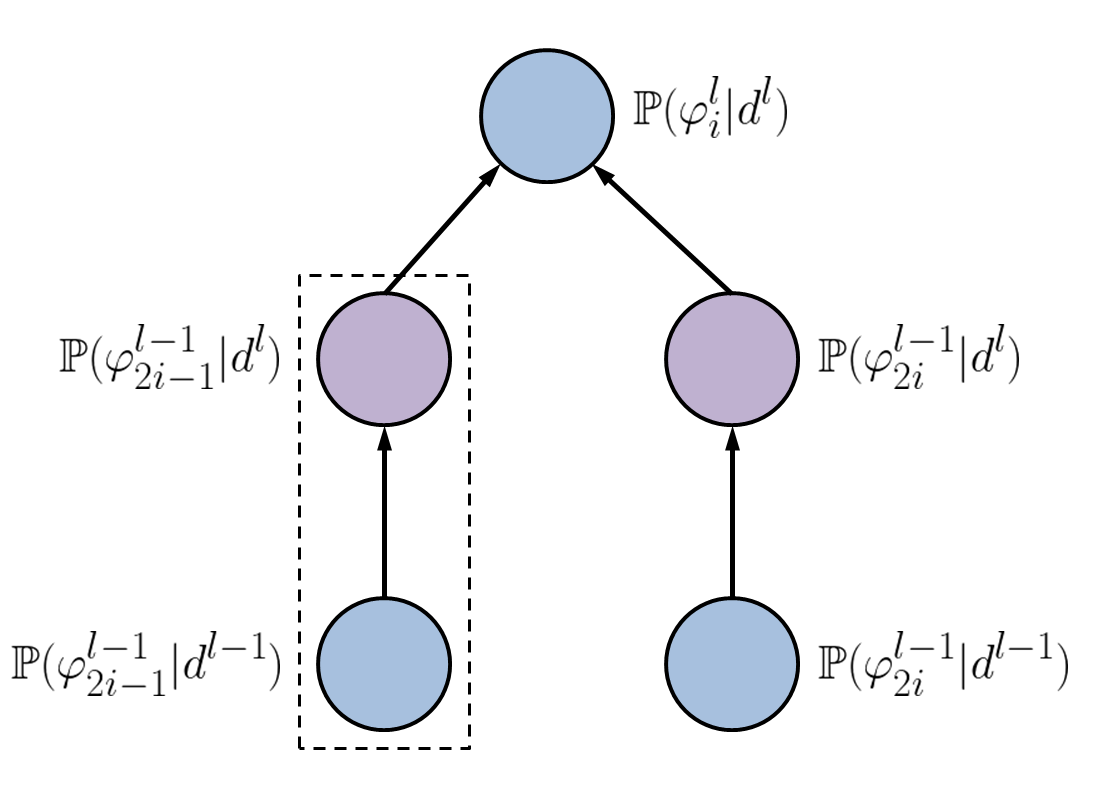}
  	\caption{The refined structure of a fusion.}
  	\label{fig3}
  	\vspace{-0.17in}
  \end{figure}
 
 If matrix $\bm{\mathrm{A}}$ is invertible, we can get: 
 $$\bm{\mathrm{v}} = \bm{\mathrm{A}}^{-1} \cdot \bm{\mathrm{u}},$$
 and by the Definition \ref{100}, we can get:
 \begin{equation}
 H(\varphi_{2i-1}^{l-1}|d^{l-1}) \geq H(\varphi_{2i-1}^{l-1}|d^{l}),  \label{15}
 \end{equation}
combine (\ref{14}) and (\ref{15}):
 \begin{equation}
H(\varphi_{2i-1}^{l-1}|d^{l}) = H(\varphi_{2i-1}^{l-1}|d^{l-1}) ,
 \end{equation}
 i.e. if matrix $\bm{\mathrm{A}}$ is invertible, there is no information loss. So we can propose a reasonable assumption:

 \begin{assumption}
 For (\ref{14}), when $\bm{\mathrm{A}}$ is  a singular matrix, there exists maximum information loss. i.e. the uncertainty of $\bm{\mathrm{u}}$ is larger than $\bm{\mathrm{v}}$. And the information entropy increases with the increase of uncertainty. So we can get the following fomular:
\begin{equation}
H(u) - H(v) \leq C, \label{20}
\end{equation}
and another fomular for the same meaning:
\begin{equation}
\frac{H(u)}{H(v)} \leq \beta, \label{21}
\end{equation}
where $H(u)$, $H(v)$ is the information entropy of $u$ and $v$, respectively. And $C$ presents the maximum information entropy increase, and $\beta\!>\!1$ presents the maximum gain ratio of information entropy.
 \end{assumption}
  \begin{figure}[t]
 	\centering
 	\includegraphics[totalheight=5.0cm]{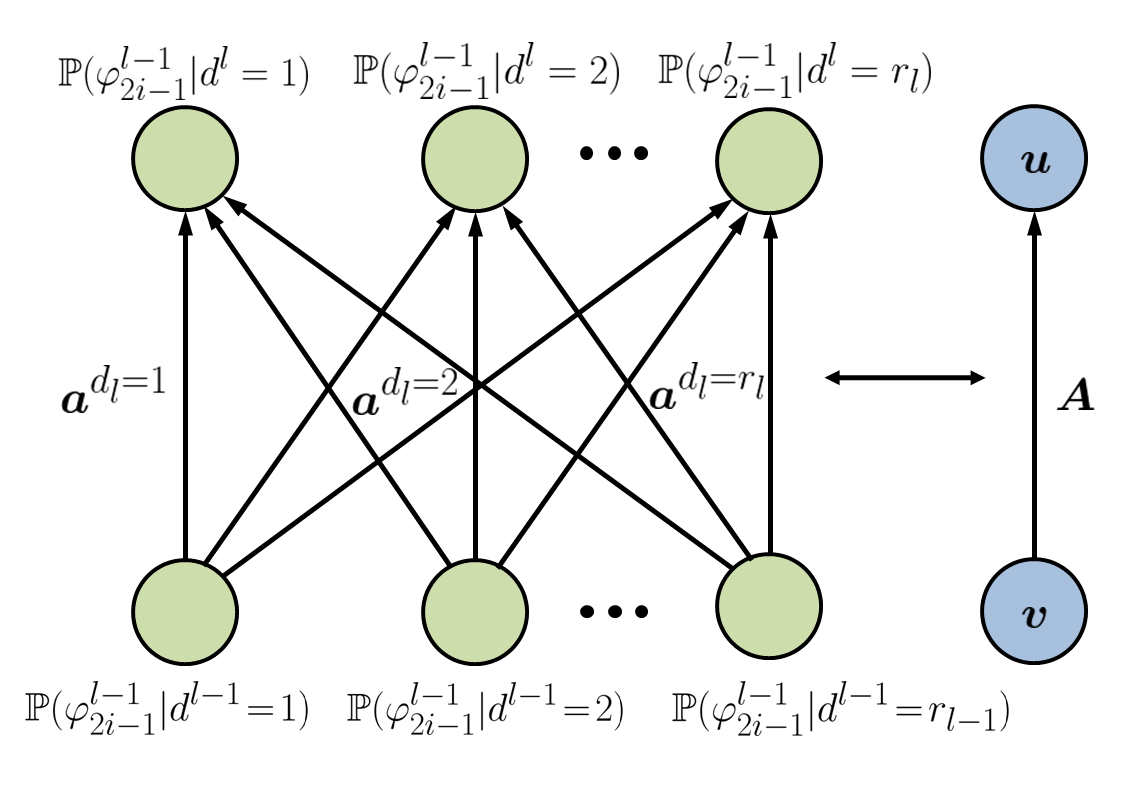}
 	\caption{The more refined structure of the dotted box shown in Fig. \ref{fig3}.}
 	\label{fig4}
 	\vspace{-0.17in}
 \end{figure}
 We now present that the information scaling law of the ConvACs.
 \begin{theorem}
Information scaling law of HT-model:
\begin{equation}
%\begin{split}
H(X) - \sum_{i=1}^{N}H(\bm{\mathrm{x}}_i|d) \leq (2^L +2^{L-1} +...+ 1)C,  \label{18}
%\end{split}
\end{equation}
\begin{equation}
\frac{H(X)}{\sum_{i=1}^{N}H(\bm{\mathrm{x}}_i|d)} \leq \beta^{L+1},
\end{equation}
where $H(X)$ is the information entropy of $X$, and $H(\bm{\mathrm{x}}_i|d)$ is the information entropy of local strcture $\bm{\mathrm{x}}_i$. $C$ is the maximum information entropy increase of a mapping of probability space, and $\beta$ is the maximum gain ratio of information entropy of a mapping of probability space.
 \end{theorem}
\begin{proof}
As mentioned before, a fusion could be descirbed by (\ref{9})(\ref{10})(\ref{8}). By the Assumption 1, from (\ref{10})(\ref{8}), we can get:
\begin{equation}
\begin{split}
H(\varphi_{2i-1}^{l-1}|d^{l}) - H(\varphi_{2i-1}^{l-1}|d^{l-1}) \leq C,\\\label{13}
H(\varphi_{2i}^{l-1}|d^{l}) -  H(\varphi_{2i}^{l-1}|d^{l-1}) \leq C,    \\    
\end{split}
\end{equation}
\begin{equation}
\begin{split}
\frac{H(\varphi_{2i-1}^{l-1}|d^{l})}{H(\varphi_{2i-1}^{l-1}|d^{l-1})} \leq \beta,\\
\frac{H(\varphi_{2i}^{l-1}|d^{l})}{H(\varphi_{2i}^{l-1}|d^{l-1})} \leq \beta.  \label{39}
\end{split}
\end{equation}

From (\ref{9}), by Definition 4, we can get:
\begin{equation}
H(\varphi_i^l|d^l) = H(\varphi_{2i-1}^{l-1}|d^{l}) + H(\varphi_{2i}^{l-1}|d^{l}). \label{16}
\end{equation}
Combining (\ref{39}) and (\ref{16}), we can get the information scaling law of a fusion:
\begin{equation}
H(\varphi_i^l|d^l)\! - \!(H(\varphi_{2i-1}^{l-1}|d^{l-1})\! +\! H(\varphi_{2i}^{l-1}|d^{l-1})) \leq 2C,  \label{17}
%\frac{H(\varphi_{i}^{l}|d^{l}))}{H(\varphi_{2i-1}^{l-1}|d^{l-1})+H(\varphi_{2i}^{l-1}|d^{l-1})} \leq \beta. \label{40}
\end{equation}
\begin{equation}
%H(\varphi_i^l|d^l)\! - \!(H(\varphi_{2i-1}^{l-1}|d^{l-1}))\! +\! H(\varphi_{2i}^{l-1}|d^{l-1}) \leq 2C, \\  \label{17}
\frac{H(\varphi_{i}^{l}|d^{l}))}{H(\varphi_{2i-1}^{l-1}|d^{l-1})+H(\varphi_{2i}^{l-1}|d^{l-1})} \leq \beta. \label{40}
\end{equation}
As shown in Fig. \ref{fig2}, the ConvACs has $L = \log_2N$ layers, and at the $l$-th layer has $2^{L-l+1}$ fusions. So we get :
\begin{equation}
%\begin{split}
H(\varphi^L|d^L) - \sum_{i=1}^{N}H(\varphi_i^0|d^0) \leq (2^L \!+\!2^{L-1}\! +\!...\!+\!2)C, \label{42}
%\frac{\sum_{i=1}^{N}H(\varphi_i^0|d^0)}{H(\varphi^L|d^L)} \leq \beta^{L}
%\end{split}
\end{equation}
\begin{equation}
%\begin{split}
%H(\varphi^L|d^L) - \sum_{i=1}^{N}H(\varphi_i^0|d^0) \leq (2^L \!+\!2^{L-1}\! +\!...\!+\!2)C 
\frac{H(\varphi^L|d^L)} {\sum_{i=1}^{N}H(\varphi_i^0|d^0)}\leq \beta^{L}. \label{43}
%\end{split}
\end{equation}
We get the following equations by $ \mathbb{P}(X) = \sum_{d^L = 1}^{r_L}a_{d^L}P(\varphi^L|d^L)$ and Assumption 1:
\begin{equation}
H(X)-H(\varphi^L|d^L) \leq C, \label{44}
\end{equation}
\begin{equation}
\frac{H(X)}{H(\varphi^L|d^L)} \leq \beta. \label{45}
\end{equation}
Combining (\ref{42})(\ref{44}) and (\ref{43})(\ref{45}), respectively, we can get the information scaling law of HT-model:
\begin{equation}
H(X) - \sum_{i=1}^{N}H(\bm{\mathrm{x}}_i|d) \leq (2^L +2^{L-1} +...+2 + 1)C,
\end{equation}
\begin{equation}
\frac{H(X)}{\sum_{i=1}^{N}H(\bm{\mathrm{x}}_i|d)} \leq \beta^{L+1}.
\end{equation}
\end{proof}

From Assumption 1 and (\ref{4}), the following theorem about CP-model is a direct result of Theorem 4.

 \begin{theorem}
Information scaling law of CP-model:
\begin{equation}
%\begin{split}
H(X) - \sum_{i=1}^{N}H(\bm{\mathrm{x}}_i|d) \leq (N + 1)C,  \label{19}
%\end{split}
\end{equation}
\begin{equation}
\frac{H(X)}{\sum_{i=1}^{N}H(\bm{\mathrm{x}}_i|d)} \leq \beta^2.
\end{equation}
%where $H(X)$ is the information entropy of $X$, $H(x_i|d)$ is the information entropy of local strcture $x_i$. And $C$ is the maximum information entropy increase of a mapping of probability space,  $\beta$ is the maximum gain ratio of information entropy of a mapping of probability space.
 \end{theorem}

% \begin{figure}
%     \centering
%     \includegraphics[totalheight=4.5cm]{real_data}
%     \caption{The CDF graph of location error by TGAN network and KNN algorithm using data collected by android prototype. We randomly select 12 points as the testing points. The localization error equals 0 means the error less than minimum scale ($1\text{m}$).}
%     \label{fig:real_data}
% \end{figure}

% \begin{figure}
%     \centering
%     \includegraphics[totalheight=4cm]{real_scenario}
%     \caption{A simple demonstration of the experiment region we choose and the interface of our Android detector.}
%     \label{fig:real_scenario}
% \end{figure}

%%%%%%%%%%%%%%%%%%%%%%%%%%%%%%%%%%%%%%%

\section{Conclusion}
In this paper, we have revealed an information scaling law of DNNs. At first, we convert the complex DNNs to rigorous mathematics by exploiting the ConvACs. It is convenient for us to explore the inner organization of DNNs with proper mathematical presentation. Therefore, the information scaling law allows us to understand the complicated DNNs more concretely from the perspective of information theory and provides us designing suggestions of DNNs.

\clearpage
\bibliographystyle{named}
\bibliography{TMM_Information_qingyun}
\end{document}